\documentclass[letterpaper, 10 pt, conference]{ieeeconf}  

\IEEEoverridecommandlockouts                              

\usepackage{times}
\usepackage{epsfig}
\usepackage{graphicx}
\usepackage{amsmath}
\usepackage{amssymb}
\usepackage{booktabs}
\usepackage[table, dvipsnames, svgnames, x11names]{xcolor}
\usepackage{marvosym}
\expandafter\let\csname Cross\endcsname\relax
\usepackage{bbding}
\usepackage{array}
\usepackage{textcomp}
\usepackage{algorithm}
\usepackage{algorithmic}
\usepackage{makecell}
\usepackage{stfloats}
\usepackage{bm}
\usepackage{multirow}
\usepackage{amsfonts}
\usepackage{lipsum}
\usepackage[normalem]{ulem}
\usepackage{colortbl}
\usepackage{xfrac}
\usepackage{threeparttable}
\usepackage[colorlinks,
    linkcolor=black,
]{hyperref}
\usepackage{color}
\usepackage{cite}

\newtheorem{theorem}{Theorem}

\title{\LARGE \bf
LIO-PPF: Fast LiDAR-Inertial Odometry via\\Incremental Plane Pre-Fitting and Skeleton Tracking
}

\author{Xingyu Chen, Peixi Wu, Ge Li and Thomas H. Li
\thanks{All authors are with the School of Electronic and Computer Engineering, Peking University, China. Thomas H. Li is the corresponding author. E-mail: {\texttt{\{cxy,wupeixi\}@stu.pku.edu.cn\quad geli@ece.pku.edu.cn\quad thomas@pku.edu.cn}}.}
\thanks{\textit{Acknowledgment:} This work was supported by National Natural Science Foundation of China (No. 62172021).}
}

\begin{document}

\maketitle
\thispagestyle{empty}
\pagestyle{empty}

\begin{abstract}
As a crucial infrastructure of intelligent mobile robots, LiDAR-Inertial odometry (LIO) provides the basic capability of state estimation by tracking LiDAR scans.
The high-accuracy tracking generally involves the \textit{k}NN search, which is used with minimizing the point-to-plane distance.
The cost for this, however, is maintaining a large local map and performing \textit{k}NN plane fit for each point.
In this work, we reduce both time and space complexity of LIO by saving these unnecessary costs.
Technically, we design a plane pre-fitting (PPF) pipeline to track the basic skeleton of the 3D scene.
In PPF, planes are not fitted individually for each scan, let alone for each point, but are updated incrementally as the scene `flows'.
Unlike \textit{k}NN, the PPF is more robust to noisy and non-strict planes with our iterative Principal Component Analyse (iPCA) refinement.
Moreover, a simple yet effective sandwich layer is introduced to eliminate false point-to-plane matches.
Our method was extensively tested on a total number of 22 sequences across 5 open datasets, and evaluated in 3 existing state-of-the-art LIO systems.
By contrast, LIO-PPF can consume only 36\% of the original local map size to achieve up to 4$\times$ faster residual computing and 1.92$\times$ overall FPS, while maintaining the same level of accuracy.
We fully open source our implementation at \url{https://github.com/xingyuuchen/LIO-PPF}.
\end{abstract}

\mathchardef\mhyphen="2D
\definecolor{aug_our_method_color}{RGB}{0,131,9}
\definecolor{X_color}{RGB}{255,0,0}
\definecolor{checkmark_color}{RGB}{0,176,80}
\definecolor{shallow_grey}{RGB}{170,170,170}

\section{Introduction}\label{sec:intro}

Estimating the ego-motion and meanwhile building the 3D map of the environment (SLAM) is a fundamental tool for intelligent mobile robots, which is requisite for many downstream applications such as route planning and obstacle avoidance.
Among all the sensors, 3D LiDARs provide accurate, long-range and light-invariant detection of the environment;
Inertial Measurement Unit (IMU) measures the robot's movement at high frequency and can therefore compensate each LiDAR point's motion.
As a result, these two sensors are widely deployed together on mobile robots.

\begin{figure}[htbp]
    \centerline{\includegraphics[width=\columnwidth]{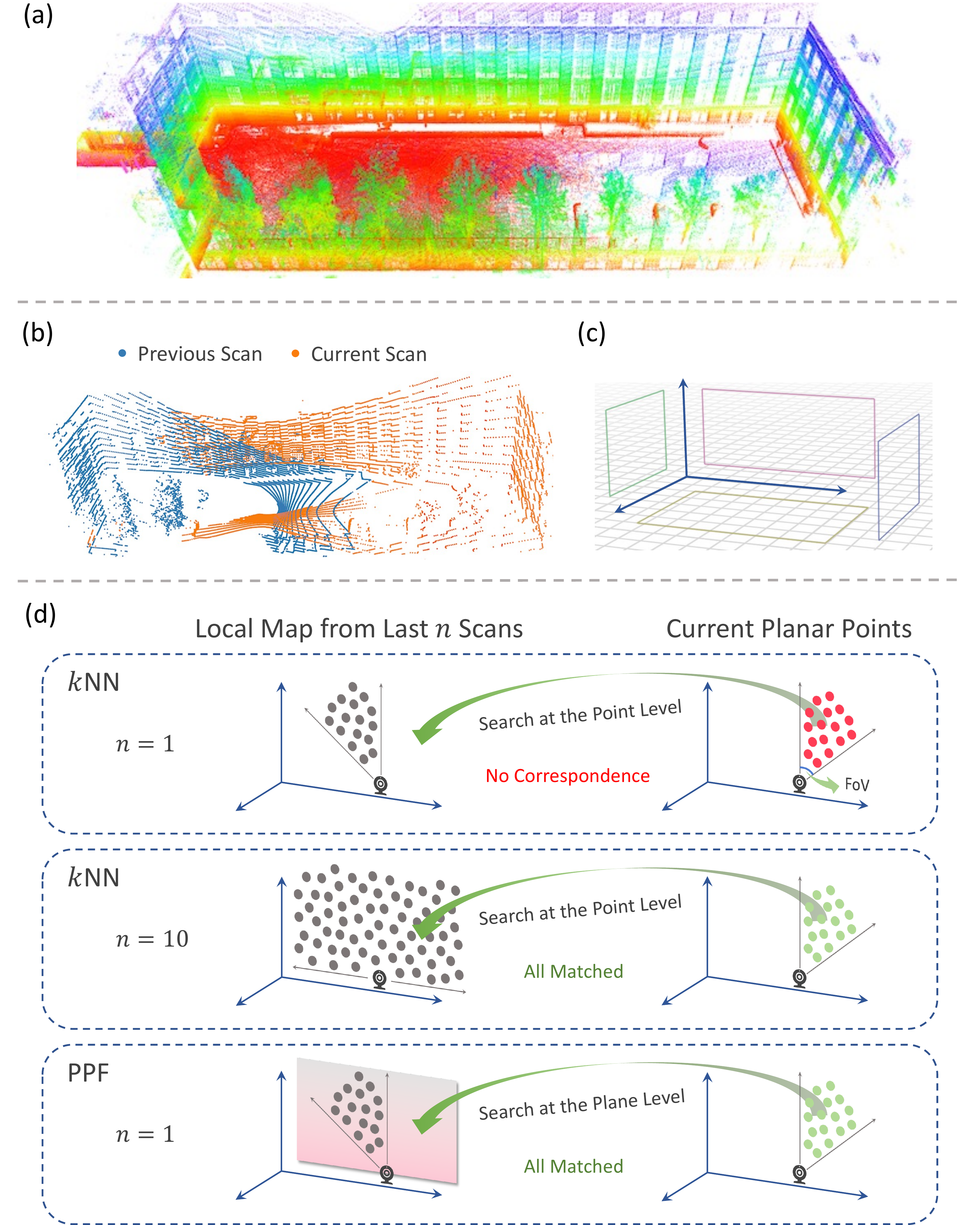}}
    \caption{\textbf{(a)} The reconstructed scene. \textbf{(b)} Two consecutive scans of (a) under rapid rotations. \textbf{(c)} Large planes across the scans form the \textit{basic skeleton} of the scene and reveal its overall geometry structure. \textbf{(d)} A large local map is needed by \textit{k}NN otherwise most points can't find neighbors to fit planes. While we use the basic skeleton to represent the scene for point matching. The search field is enlarged to plane level with no \textit{k}NN search required.}
    \label{fig:plane_win_size}
\end{figure}

Though deep learning methods have made great progress and sometimes outperform traditional LIO, not to mention the poor generalizability across different datasets, the high power of modern GPUs is still too expensive, which is not applicable to most battery-powered mobile agents.
Moreover, in typical intelligent robots, LIO runs only at the low level of the overall system, where high-level tasks (\textit{e.g.} route planning) can take up most computational resources.
Therefore, all infrastructures should be guaranteed to run robustly and efficiently.
For this, this paper focuses on reducing \textit{both} the time and space complexity of the low-level LIO systems.

Modern real-time LIO solves the ego-motion by registering LiDAR scans~\cite{zhang2014loam,xu2021fast,xu2022fast,bai2022faster,shan2020lio,liliom2021}.
The registration is based on the geometry constraints, \textit{i.e.} point-to-plane distance~\cite{bai2022faster,xu2021fast,xu2022fast} (and sometimes point-to-edge distance~\cite{shan2020lio,liliom2021}).
To match each point to its corresponding plane, the nearest-neighbor (NN) search is performed, and the local plane patch is then fitted from \textit{k} (4$\sim$5) NN.
One strategy is to register consecutive scans~\cite{zhang2014loam}, namely, scan-to-scan.
Though fast, this strategy can easily end up with few matched points caused by small FoV or aggressive rotational maneuvers, \textit{e.g.} Fig.~\ref{fig:plane_win_size}d top.
Therefore, recent methods registered the latest LiDAR scan to the local map~\cite{xu2021fast,xu2022fast,bai2022faster,shan2020lio,liliom2021}, which consists of several selected scans transformed to the same coordinate system, thereby providing more points for \textit{k}NN search.
This ensures sufficient matched points and high accuracy.
However, maintaining a large local map, performing \textit{k}NN search and plane fitting for each point sacrifice both memory and efficiency.

In this paper, compared to \textit{lazily} fitting the plane after searching \textit{k}NN for each point, we adopt a more \textit{active} plane pre-fitting manner.
This idea stems from two \textit{un-necessities:}

\textit{1) Unnecessary \textit{k}NN.}
Current LiDAR scan tracking~\cite{zhang2014loam,xu2021fast,xu2022fast,bai2022faster,shan2020lio,liliom2021} search \textit{k}NN for every single point to fit a plane, assuming that spatially close points are from the same plane.
However, the condition is too strict, \textit{i.e.}, spatially distant points can also belong to the same plane, \textit{e.g.} long walls.
In other words, the \textit{k}NN strategy neglects the fact that planes are not always in small local areas.
For large planes, most \textit{k}NN searches are redundant as they turn out to fit the same plane.
Meanwhile, points from successive LiDAR scans often come from different parts of the same large plane, \textit{i.e.}, they share the same high-level large planar object.
This gives us a hint: \textit{If searching happens in the high-level feature plane space, the expensive costs from \textit{k}NN could naturally be saved}.

\textit{2) Unnecessary large local map.}
If preserving only a few scans in the local map, NN might be far away and thus likely to belong to other objects, resulting in the \textit{k}NN being inaccurate to fit the matching plane.
Hence, the local map is enlarged to guarantee sufficient close-enough point-to-\textit{k}NN correspondences~\cite{shan2020lio,liliom2021,qin2020lins}, see Fig.~\ref{fig:plane_win_size}d.
However, we note that: Even with only one scan kept in the local map, the absence of close NN $\neq$ the absence of plane match.
\textit{E.g.}, in Fig.~\ref{fig:plane_win_size}b, most orange wall points failed to find close NN on the long wall, the plane of which, however, can be easily \textit{pre}-fitted from the points within one single previous scan (blue points).
Again, this gives us a hint: \textit{If searching happens in the high-level feature plane space, one single previous scan is enough to provide adequate point-to-plane matches}.

In summary, we make the following contributions:
\begin{itemize}
    \item We propose to represent the scene by its basic skeleton for point matching, which is done via plane pre-fitting (PPF) across LiDAR scans.
    We demonstrate (\textit{\romannumeral1}) PPF can naturally downsize the local map and eliminate most of the redundant \textit{k}NN searches and plane fittings;
    (\textit{\romannumeral2}) the un-robustness of \textit{k}NN to noisy and non-strict planes.

    \item In PPF, planes are not fitted for each point/scan, rather, we leverage IMU and the sequential nature of LiDAR scans to achieve incremental plane updating;
    The iterative PCA makes PPF more robust to noisy and non-strict planes than \textit{k}NN;
    A simple yet effective sandwich layer is introduced to exclude false point-to-plane matches.

    \item Engineering contributions include the fully open-source LIO-PPF.
    Experiments on 22 sequences across 5 open datasets show that, compared to the originals, our PPF reduces the local map size by at most 64\%, achieving 4$\times$ faster in residual calculating, up to 1.92$\times$ overall FPS, and still shows the same level of accuracy.
\end{itemize}

\section{Related Work}\label{sec:related_work}

LOAM~\cite{zhang2014loam} pioneered the field of modern 3D LiDAR odometry and mapping.
It first extracts two kinds of feature points (\textit{i.e.} edge and plane) based on the local curvature (roughness).
Then, the feature cloud is processed in two separate algorithms, which were derived from the NN-based feature-matching scheme described in~\cite{zhang2017low}.
First, the LiDAR odometry took two consecutive scans as input and computed a rough motion in real time.
Second, the LiDAR mapping registered the undistorted points onto the global map at a lower frequency (around 1Hz).
LeGO-LOAM~\cite{shan2018lego} is a lightweight and ground-optimized version of LOAM, which extracted ground points to obtain $t_z$, $\theta_{roll}$ and $\theta_{pitch}$ first.
Since the scan-to-scan matching is restricted by small FoV and hence few correspondences can be matched,~\cite{shan2020lio} introduced the concept of sub-keyframes (\textit{a.k.a.} local map in~\cite{liliom2021,ye2019tightly,qin2020lins}), which consists of several selected LiDAR scans.
The \textit{k}NN search is performed in the local map to guarantee sufficient matches.
Among these tightly-coupled methods, LIO-Mapping~\cite{ye2019tightly} introduced a rotation-constraint refinement to align the final LiDAR pose to the global map;
LIO-SAM~\cite{shan2020lio} is the one with the most impressive mapping results.
~\cite{zhou2021lidar} solved the `double-side issue' of planar objects, and~\cite{zhou2022mathcal} further exploited the lines and cylinders.

Even with parallel computing, \textit{k}NN searching still requires a time-consuming kd-tree building.
Therefore, Fast-LIO2~\cite{xu2022fast} proposed to incrementally build the kd-tree (namely, ikd-tree~\cite{cai2021ikd}), which saves the time of building from scratch each time.
As the follow-up of Fast-LIO2, Faster-LIO~\cite{bai2022faster} regarded the strict \textit{k}NN search as unnecessary in most cases.
Hence, an incremental voxel-based local map (iVox) was designed to be an alternative \textit{k}NN structure to the ikd-tree.

Plane segmentation of 3D point clouds has also been widely studied.
Using region growing~\cite{farid2015region,wu2019accurate} and Hough transform~\cite{hulik2014continuous,leng2016multi} are the most popular, but most are too heavy for the real-time LIO, even with coarse-to-fine refinement~\cite{oehler2011efficient}.
A more efficient faction is RANSAC~\cite{fischler1981random}-based methods.
However, they generally involved points' normal estimation~\cite{khaloo2017robust,yue2018new}, or NDT feature calculation~\cite{li2017improved,xu20193d}, which are still complex operations for LIO.
In this work, our plane pre-fitting algorithm exploits the particularity of LIO, which is efficient in handling sequential LiDAR scans with IMU measurements, with efficiencies up to $\sim$1.5 $\mathrm{ms}$ per scan (64-beam) on modern desktop CPUs.

\section{Notations and Preliminary}\label{sec:preliminary}
Throughout this paper, we denote the world frame as $\mathbf{W}$, the LiDAR frame as $\mathbf{L}$, and the IMU body frame as $\mathbf{B}$.
Let $\mathbf{P}_i$ be the de-skewed LiDAR points received from scan $i$.
A 3D plane can be determined by its normal $\mathbf{n} \in \mathbb{R}^3$ and an arbitrary point $\mathbf{p} \in \mathbb{R}^3$ on it, or more compactly by a 4d vector $\bm{f} = \left[ \mathbf{n}^{\top},d \right]^{\top}$ such that $\bm{f}\cdot \tilde{\mathbf{p}}=0$, where $\tilde{\mathbf{p}} = \left[ \mathbf{p}^{\top},1 \right]^{\top}$.
We represent the 6-DoF pose by a transformation matrix $\mathbf{T} \in SE(3)$, which contains a rotation matrix $\mathbf{R} \in SO(3)$ and a translation vector $\mathbf{t} \in \mathbb{R}^3$.
\begin{theorem}
    The transformation of $\bm{f}$ by $\mathbf{T}$ is given by:
    \begin{equation}\label{eq:plane_trans}
    \bm{f}^{\prime} = \left( \mathbf{T}^{-1} \right)^{\top} \bm{f}.
    \end{equation}
\end{theorem}

\begin{proof}
    Let $\tilde{\mathbf{p}}$ be an arbitrary point on $\bm{f}$, and $\tilde{\mathbf{p}}^{\prime}$ be its transformed counterpart such that $\tilde{\mathbf{p}}^{\prime} = \mathbf{T}\tilde{\mathbf{p}}$.
    Then, we have
    \begin{equation}\label{eq:proof_plane_trans}
    \bm{f}^{\prime} \cdot \tilde{\mathbf{p}}^{\prime}
    = \bm{f}^{\prime \top} \tilde{\mathbf{p}}^{\prime}
    = \bm{f}^{\top} \mathbf{T}^{-1} \mathbf{T} \tilde{\mathbf{p}}
    = \bm{f}^{\top} \tilde{\mathbf{p}}
    = \bm{f} \cdot \tilde{\mathbf{p}}
    = 0,
    \end{equation}
    which means $\tilde{\mathbf{p}}^{\prime}$ is on $\bm{f}^{\prime}$.
\end{proof}

\begin{figure}[htbp]
    \centering\includegraphics[width=\columnwidth]{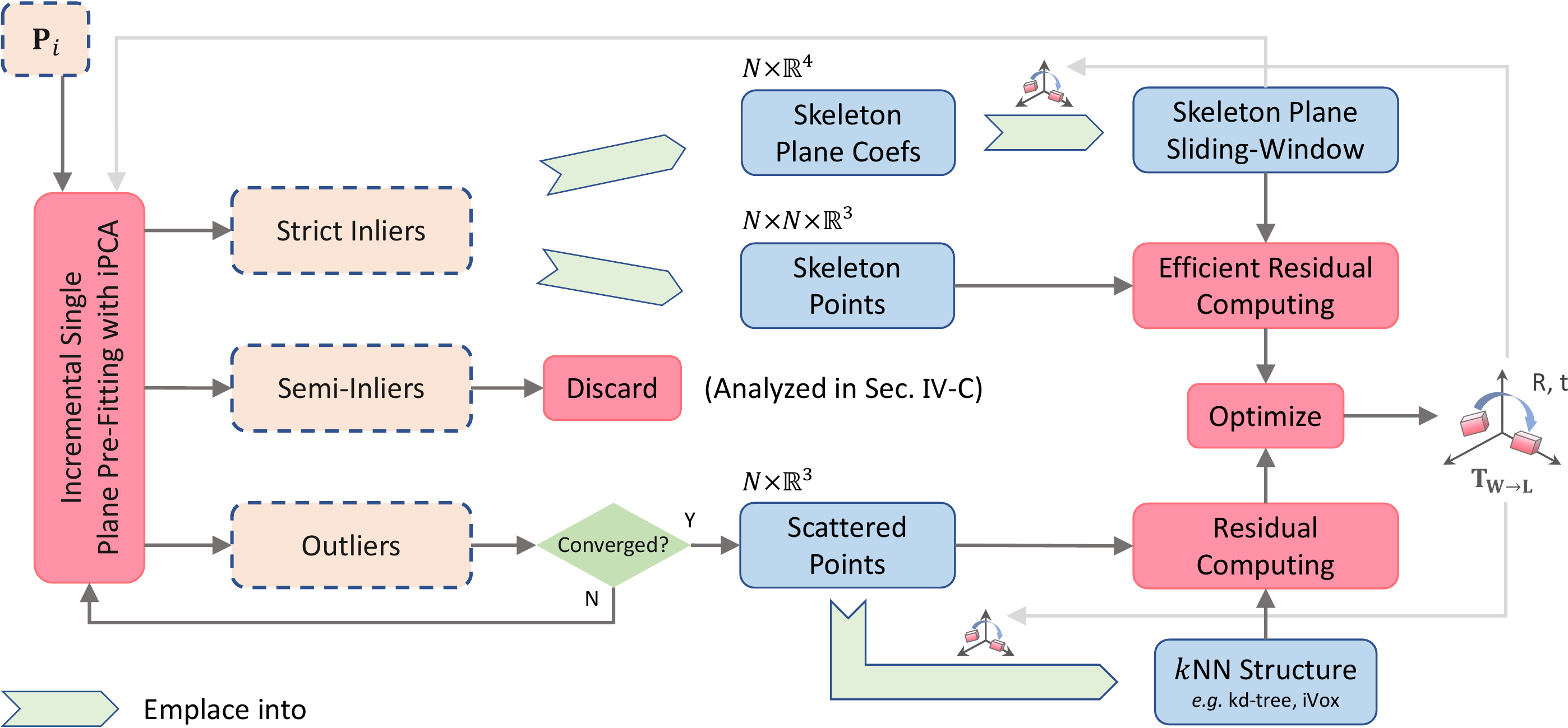}
    \caption{Overview of the plane pre-fitting and tracking pipeline. Shallow arrows \textcolor{shallow_grey}{$\rightarrow$} represent the operations to prepare for the next LiDAR scan.}
    \label{fig:overview}
\end{figure}

\section{Incremental Plane Pre-Fitting and Tracking}\label{sec:methodology}

In this section, we first introduce our efficient plane pre-fitting method.
Then, we describe the corresponding skeleton tracking algorithm based on PPF.
Finally, a sandwich layer is proposed to robustify the algorithm in complex scenes.

\subsection{Incremental iPCA Plane Pre-Fitting}\label{subsec:plane-extraction}
The core of our method can be described from 3 aspects: \textit{Pre-Fitting}, \textit{Iterative PCA} and \textit{Incremental}.

\subsubsection{Pre-Fitting}
The large surfaces across each scan form the basic skeleton of the 3D scene and reveal its global geometry feature.
We refer to them as the skeleton planes, which are independently maintained instead of \textit{lazily} fitted by \textit{k}NN.
We first take input from the extracted surface feature points~\cite{shan2020lio} or the raw de-skewed cloud~\cite{bai2022faster,xu2022fast}.
Then, we adopt an iterative \textit{fit-remove-fit} manner to extract all skeleton planes in $\mathbf{P}_i$.
As for single plane fitting, we use the classic RANSAC~\cite{fischler1981random} and NAPSAC~\cite{torr2002napsac}\footnote[1]{\,Both are integrated into our open-source algorithm.}.
The NAPSAC assumes that samples within the same class are closer to each other than points outside the class.
Hence, samples from inside the same hyper-sphere are more likely to fit the best candidate model, but again, this requires the expensive \textit{k}NN to guarantee.
Thus, in most cases, simply increasing the RANSAC iterations can reach the same level of accuracy.
Once a skeleton plane is fitted, all points are divided into \textit{strict inliers}, \textit{semi-inliers} (Sec.~\ref{subsec:sandwich-layer}) and \textit{outliers} according to their distance to the plane.
All semi-inliers are directly discarded (analyzed later), and the outliers are fed into this process again.
Since the skeleton tracking is effective for large planes, the loop terminates when no more qualified planes can be extracted or the remaining points are too few.\footnote[2]{~Procedures of point classification, semi-inliers removal, \textit{etc.} are carefully engineered with OMP~\cite{dagum1998openmp} parallel computing. See \href{https://github.com/xingyuuchen/LIO-PPF/blob/master/include/sac_model_plane.h}{\color{magenta}{sac\_model\_plane.h}}.}

\subsubsection{Iterative PCA Refining}
Applying PCA to \textit{k}NN dates back to~\cite{zhang2017low}, with the only purpose of determining whether more than 3 points can fit a plane.
For the PPF, it can be more useful.
Due to the noise of LiDAR sampling, IMU noise, bias and its random walk, de-skewed point clouds are noisy.
\begin{figure}[htbp]
    \centerline{\includegraphics[width=\columnwidth]{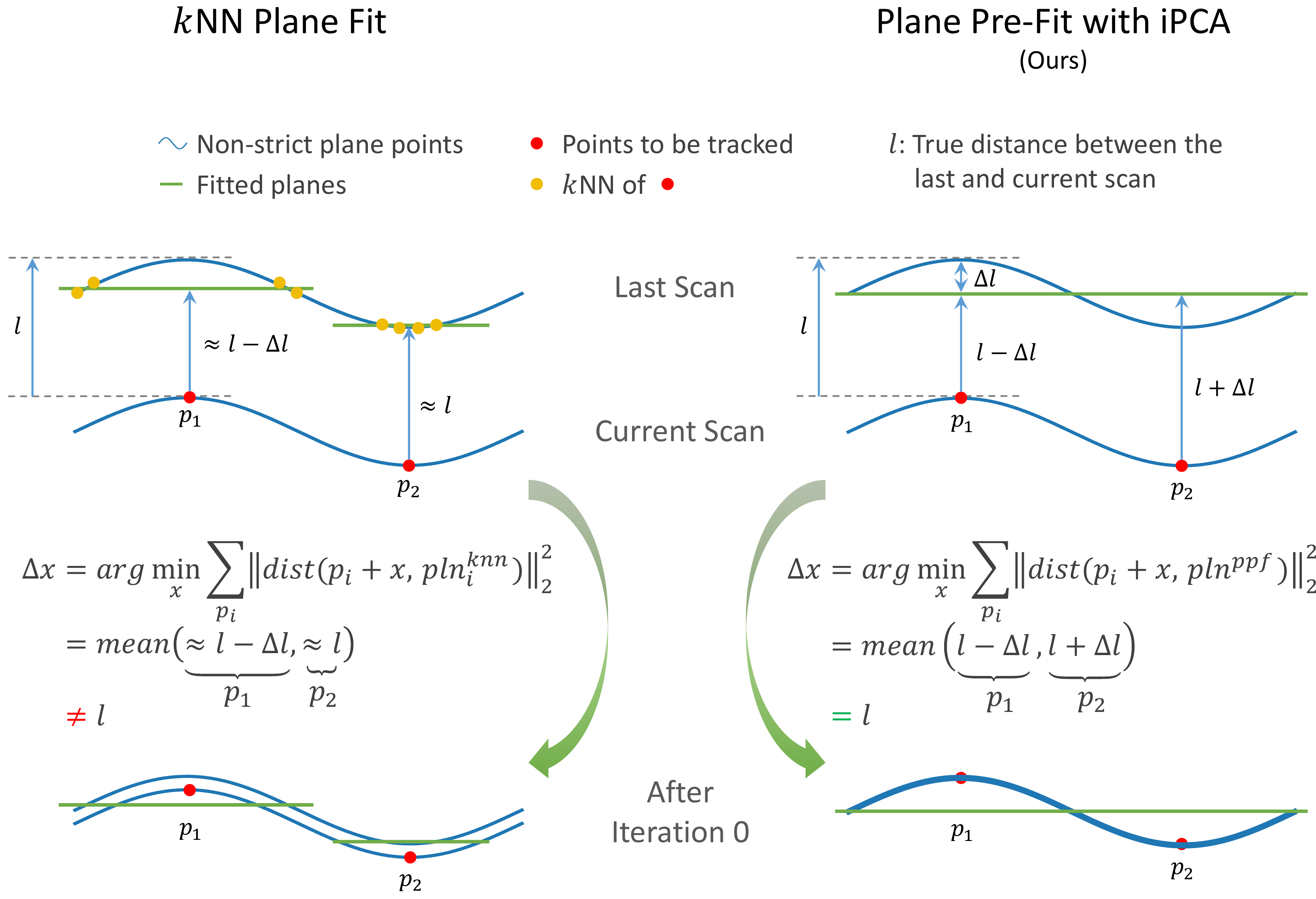}}
    \caption{Illustration of the \textit{k}NN degradation when tracking non-strict planes with small local curvature. The \textit{k}NN strategy minimizes the point-to-\textit{local}-plane distances. However, the local space determined by 4 NN cannot reflect the overall geometric forms. In contrast, the iPCA manner iteratively extracts the principal skeleton of all \textit{global} planar points, resulting in faster convergence. Experimental validation in Sec.~\ref{subsec:simulated-experiment}.}
    \label{fig:compare_knn_plane}
\end{figure}
Even in strict real-world planes (\textit{e.g.} walls), non-negligible offsets between points and the corresponding plane still exist.
On the other hand, some flat surfaces which are not strict planes can also have low local curvature (\textit{e.g.} uneven grounds, small-angle arcs).
Fig.~\ref{fig:compare_knn_plane} demonstrates the issue of \textit{k}NN degradation in these cases.
The \textit{k}NN strategy uses only 4$\sim$5 NN, meaning that the fitted plane is determined by a small local space.
When encountering large noise or non-strict planes, the local space cannot characterize the overall forms, thereby misleading the optimizing objective of the registration (see Fig.~\ref{fig:compare_knn_plane} left).
On the contrary, the \textit{pre}-fitting manner can make use of all planar points' spatial information to produce their global skeleton plane.
In particular, based on the current fitted plane, the eigen-decomposition is performed on the covariance matrix of all its \textit{strict inliers}, and the resulting eigenvector associated with the smallest eigenvalue corresponds to the refined plane, which can be understood as newly voted by \textit{all} the points that strictly support the old plane.
The new strict inliers are computed accordingly and the process is looped 3 times~\cite{Rusu_ICRA2011_PCL}.
We draw a conclusion in Fig.~\ref{fig:compare_knn_plane} that this iterative refinement brings PPF a faster convergence than \textit{k}NN, which is further validated by more detailed comparative experiments in Sec.~\ref{subsec:simulated-experiment}.

\subsubsection{Incremental Fitting}
The geometry structures of adjacent LiDAR scans do not vary a lot, instead, they are generally similar.
Inspired by this property, skeleton planes can be updated \textit{incrementally} so that unnecessary RANSAC iterations will be saved.
We first integrate (discrete-time, first-order hold) the IMU measurements between two consecutive LiDAR scans by iterating:
\begin{figure}[htbp]
    \centerline{\includegraphics[width=\columnwidth]{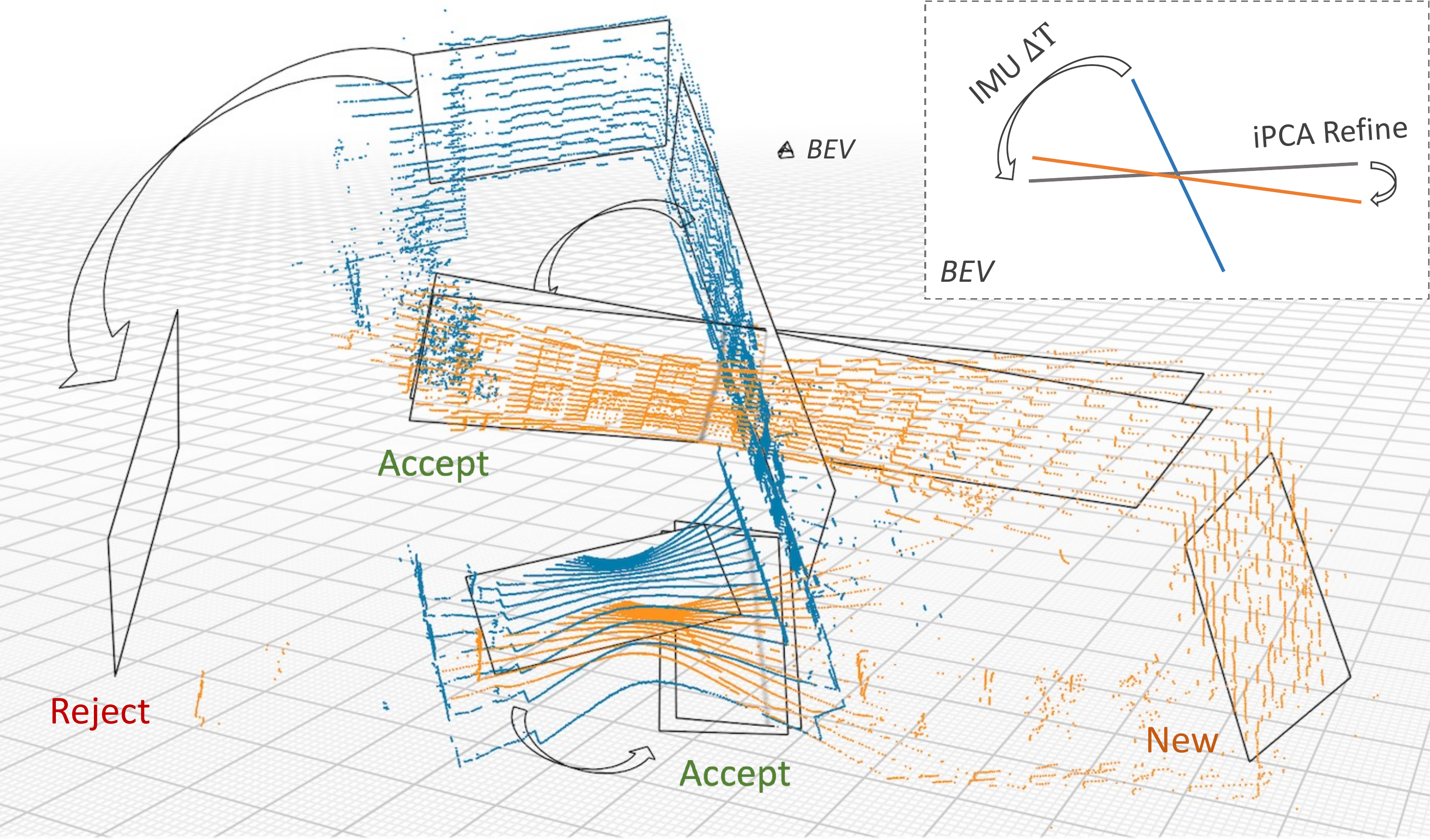}}
    \caption{Illustration of the incremental fitting for the two scans in Fig.~\ref{fig:plane_win_size}b (before scan-matching). \textit{B}ird's-\textit{E}ye-\textit{V}iew on top right. We discard the \textit{initial skeleton guesses} that failed the inliers check. For those accepted, a simple iPCA refinement is performed. Only a few are newly fitted from RANSAC.}
    \label{fig:incremental_fitting}
\end{figure}
\begin{equation}\label{eq:midpoint_w}
    \boldsymbol{\omega} = \frac{1}{2} \left( \left( \hat{\boldsymbol{\omega}}_k - \mathbf{b}^{\boldsymbol{\omega}}_k \right) + \left( \hat{\boldsymbol{\omega}}_{k+1} - \mathbf{b}^{\boldsymbol{\omega}}_k \right) \right)
\end{equation}
\begin{equation}\label{eq:imu_preint_R}
    \mathbf{R}^{\mathbf{B}_i}_{\mathbf{B}_{k+1}} = \mathbf{R}^{\mathbf{B}_i}_{\mathbf{B}_{k}} \mathbf{exp}(\boldsymbol{\omega} \Delta t)
\end{equation}
\begin{equation}\label{eq:midpoint_a}
    \mathbf{a} = \frac{1}{2} \left( \mathbf{R}^{\mathbf{B}_i}_{\mathbf{B}_{k}}\left( \hat{\mathbf{a}}_k - \mathbf{b}^{\mathbf{a}}_k \right) + \mathbf{R}^{\mathbf{B}_i}_{\mathbf{B}_{k+1}}\left( \hat{\mathbf{a}}_{k+1} - \mathbf{b}^{\mathbf{a}}_k \right) \right)
\end{equation}
\begin{equation}\label{eq:imu_preint_alpha}
    \mathbf{p}^{\mathbf{B}_i}_{\mathbf{B}_{k+1}} = \mathbf{p}^{\mathbf{B}_i}_{\mathbf{B}_{k}} + \mathbf{v}^{\mathbf{B}_i}_{\mathbf{B}_{k}} \Delta t + \frac{1}{2}\mathbf{a} \Delta t^2 + \frac{1}{2}\mathbf{R}_{\mathbf{B}_i}^{\mathbf{W} \top}\mathbf{g}\Delta t^2
\end{equation}
\begin{equation}\label{eq:imu_preint_beta}
    \mathbf{v}^{\mathbf{B}_i}_{\mathbf{B}_{k+1}} = \mathbf{v}^{\mathbf{B}_i}_{\mathbf{B}_{k}} + \mathbf{a} \Delta t + \mathbf{R}_{\mathbf{B}_i}^{\mathbf{W} \top}\mathbf{g}\Delta t,
\end{equation}
where $k \in [i,j]$ is the IMU sampling instant between consecutive LiDAR sampling instants $i$ and $j$.
$\Delta t$ is the IMU sampling interval.
$\hat{\boldsymbol{\omega}}_k$ and $\hat{\boldsymbol{a}}_k$ are the raw IMU angular velocity and acceleration measurements at time $k$ in $\mathbf{B}$.
IMU gyroscope bias $\mathbf{b}^{\boldsymbol{\omega}}$ and accelerator bias $\mathbf{b}^{\boldsymbol{a}}$ are modeled as random walk~\cite{qin2018vins}, whose derivatives are Gaussian white noise.
$\mathbf{g}$ is the gravity vector in $\mathbf{W}$.
Please refer to~\cite{forster2016manifold} for detailed derivations of Eqs.~(\ref{eq:midpoint_w}-\ref{eq:imu_preint_beta}).

Next, for every incoming LiDAR scan $\mathbf{P}_j$, each current skeleton plane $\bm{f}_{\mathbf{L}_i}$ is used to obtain a corresponding \textit{initial skeleton guess} $\bm{f}_{init}$ in $\mathbf{P}_j$ by:
\begin{equation}\label{eq:plane_init_guess_trans}
    \bm{f}_{init} = \bm{f}_{\mathbf{L}_j} = \left( \Delta \mathbf{T}^{-1} \right)^{\top} \bm{f}_{\mathbf{L}_i},~where
\end{equation}
\begin{equation}\label{eq:plane_init_guess_trans1}
\Delta \mathbf{T} =
\begin{bmatrix}
    \mathbf{R}^{\mathbf{L}}_{\mathbf{B}} \mathbf{R}^{\mathbf{B}_i}_{\mathbf{B}_{j}} \mathbf{R}^{\mathbf{L}\top}_{\mathbf{B}} & \mathbf{R}^{\mathbf{L}}_{\mathbf{B}} \mathbf{p}^{\mathbf{B}_i}_{\mathbf{B}_{j}} + \mathbf{t}^{\mathbf{L}}_{\mathbf{B}}
    - \mathbf{R}^{\mathbf{L}_i}_{\mathbf{L}_{j}} \mathbf{t}^{\mathbf{L}}_{\mathbf{B}} \\
    \mathbf{0}_{1\times 3} & 1
\end{bmatrix},
\end{equation}
$\mathbf{R}^{\mathbf{L}}_{\mathbf{B}}$ and $\mathbf{t}^{\mathbf{L}}_{\mathbf{B}}$ are the extrinsics between LiDAR and IMU.
Note that the IMU-integrated $\Delta\mathbf{T}$ is not required to be very accurate, since it is only to provide a \textit{guess}.
$\bm{f}_{init}$ is accepted only if it passed the inliers checking.
Once accepted, $\bm{f}_{init}$ will also be iPCA refined for 5 iterations.
After handling the initial guesses, the naive RANSAC begins.
We emphasize that this technique is particularly effective for LiDARs with high beam numbers, \textit{e.g.}, for Velodyne HDL-64E in \texttt{KITTI} dataset, the fitting algorithm can be boosted up to 4$\times$ faster.

\subsection{Skeleton-based Tracking}\label{subsec:scan-matching}
Based on PPF, we describe the efficient residual evaluation criterion.
During the iterative fitting, after refining every skeleton plane, we emplace its \textit{strict inliers} into a set, denoted as $\mathbf{\Omega}$.
When the PPF converges, the remaining points are considered to lie scatteredly.
We denote the set containing the scattered points as $\mathbf{\Theta}$.
Note that the semi-inliers are discarded so we have $\mathbf{\Omega} \cup \mathbf{\Theta} \subseteq \mathbf{P}$.
For each skeleton point $\tilde{\mathbf{p}}$ in $\mathbf{\Omega}$, we search its correspondence in the skeleton plane space $\mathbf{F}$.
$\mathbf{F}$ is implemented as a sliding window of $\bm{f}$ in $\mathbf{W}$.
The residual is measured as point-to-nearest-plane distance:
\begin{equation}\label{eq:residual_large_plane}
    e_{\Omega} \left( \tilde{\mathbf{p}} \right) =
\begin{cases}
    0, ~~~\min_{\bm{f} \in \mathbf{F}} \left\{ \lvert\, \bm{f} \cdot \tilde{\mathbf{p}}  \,\rvert \right\} \geq \epsilon \\
    \min_{\bm{f} \in \mathbf{F}} \left\{ \lvert\, \bm{f} \cdot \tilde{\mathbf{p}} \,\rvert \right\}, ~~~otherwise.
\end{cases}
\end{equation}
While for $\tilde{\mathbf{p}}$ in $\mathbf{\Theta}$, the searching still happens at point level, \textit{i.e.} building the \textit{e.g.} kd-tree or iVox for the local map to search \textit{k} NN points, and the residual is based on the distance between $\tilde{\mathbf{p}}$ and the lazily fitted plane patch $\bm{f}_{knn}$~\cite{xu2022fast,shan2020lio}:
\begin{equation}\label{eq:residual_scatter_points}
e_{\Theta} \left( \tilde{\mathbf{p}} \right) = \lvert\, \bm{f}_{knn} \cdot \tilde{\mathbf{p}} \,\rvert.
\end{equation}

Next, the scan-matching algorithm registers the latest LiDAR scan by solving for the optimal transformation
\begin{equation}\label{eq:point_2_plane_dist}
    \min_{\mathbf{T}_{j}} \Big\{ \sum_{\tilde{\mathbf{p}} \in \mathbf{\Omega}} e_{\Omega} \left( \tilde{\mathbf{p}} \right) + \sum_{\tilde{\mathbf{p}} \in \mathbf{\Theta}} e_{\Theta} \left( \tilde{\mathbf{p}} \right) \Big\},
\end{equation}
using Levenberg-Marquardt~\cite{more1978levenberg,shan2020lio} or IEKF~\cite{xu2022fast}.
Then, in preparation for the next LiDAR scan, we transform all the newly-fitted $\bm{f}_{\mathbf{L}_{j}}$ to $\mathbf{W}$, \textit{i.e.} $\bm{f}_{\mathbf{W}} = \left( \mathbf{T}_{j}^{-1} \right)^{\top} \bm{f}_{\mathbf{L}_{j}}$, and emplace $\bm{f}_{\mathbf{W}}$ into the sliding window $\mathbf{F}$.
Note that although points in $\mathbf{\Omega}$ outnumber $\mathbf{\Theta}$, the searching space of $\mathbf{\Omega}$ (\textit{i.e.} $\mathbf{F}$) is about three orders of magnitude smaller than that of $\mathbf{\Theta}$.

\subsection{Sandwich Layer}\label{subsec:sandwich-layer}

When entering new scenes, there could be no plane match in $\mathbf{F}$ for points on the newly fitted skeleton planes.
Rather, since planes intersect, there may exist false matches if solely relying on point-to-plane distance.
This is more severe for planes with small included angles.
\begin{figure}[htbp]
    \centerline{\includegraphics[width=\columnwidth]{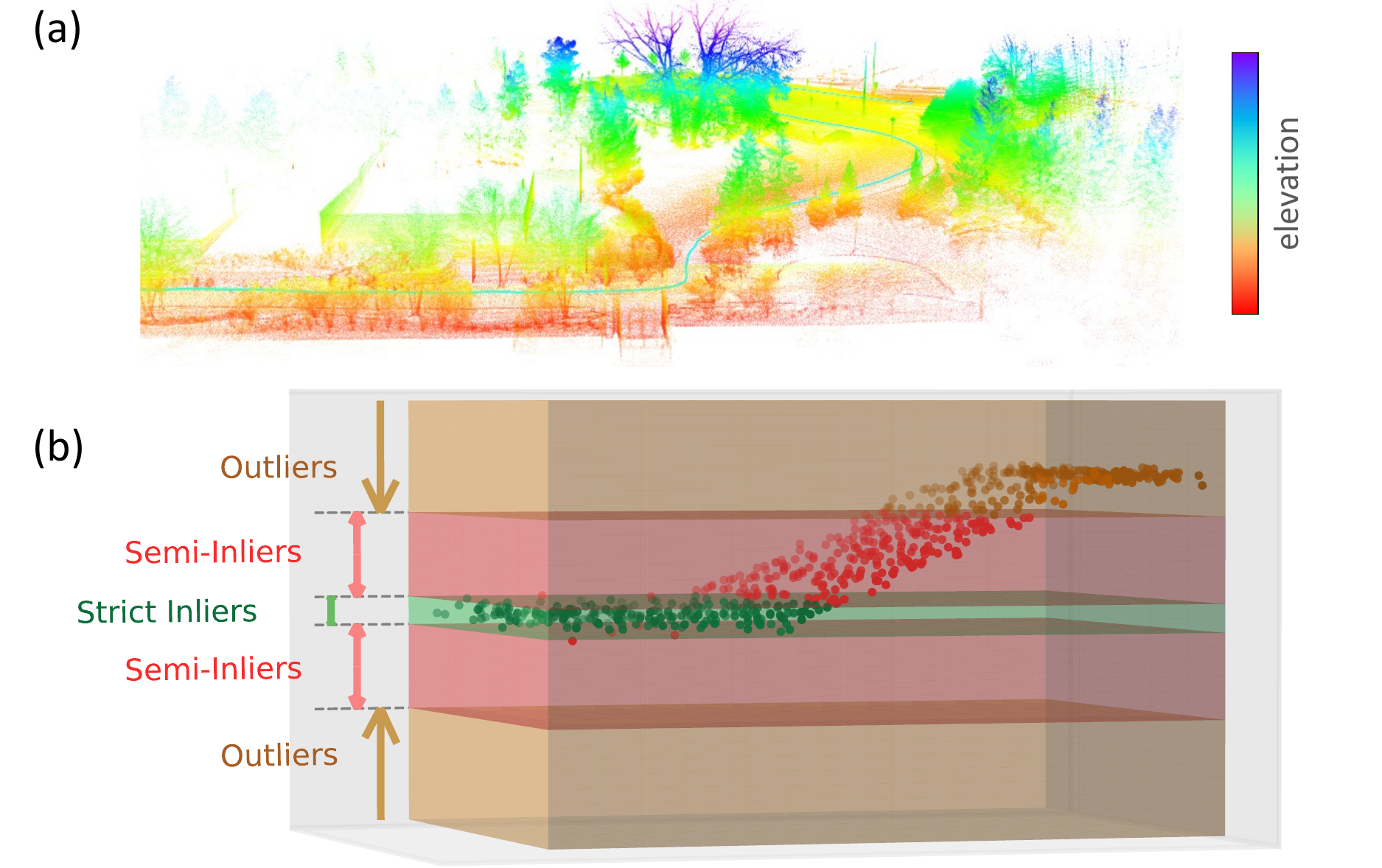}}
    \caption{\textbf{(a)} A scene with grounds of different elevations. \textbf{(b)} Schematic of the slope ground in (a) and the sandwich layer used in PPF. According to the distance to the central plane, points are categorized into three groups: \textit{strict inliers, semi-inliers} and \textit{outliers}.}
    \label{fig:sandwich}
\end{figure}
\textit{E.g.}, the complex terrain in Fig.~\ref{fig:sandwich}a, where the ground elevation slowly rises due to the slope.
Assuming this is the first scan fitting the slope plane such that it is not emplaced into $\mathbf{F}$ yet, points on the slope can easily pass the check in~\eqref{eq:residual_large_plane} and hence wrongly match the ground plane.
To handle this, we introduce an additional \textit{semi-inlier} layer (Fig.~\ref{fig:sandwich}b).
Before the scan-matching, we simply discard all \textit{semi-inlier} points, which may lead to potentially false matches.
The thickness of the \textit{semi-inlier} layer is set to be $\epsilon$ so that \textit{outliers} can hardly produce false matches.
On the other hand, this also helps prevent the noisy `thick' planar points from yielding multi skeleton planes - we only use points strictly \textit{s.t.} the central plane as skeleton.
The removal is so straightforward because the proportion of these ambiguous points is small and their negative impact far outweighs their contributions to the tracking.

\section{Experiments}\label{sec:experiments}

\subsection{Datasets and Experiment Setup}\label{subsec:dataset-and-setting}
All experiments were conducted using a modern desktop computer with Ubuntu 18.04 LTS.
We use a single Intel Core i7-9700K (3.60G Hz \texttimes~8 cores) CPU with 16 GiB memory without using GPU.
The OpenMP~\cite{dagum1998openmp} library is adopted for parallel computing.
For the sake of fair comparisons with previous methods, in all experiments, we launched the same number of threads in parallel as the one to be compared.

We first conducted comparative experiments to analyse the PPF registration strategy and the \textit{k}NN-based one.
Then, we extensively evaluate the accuracy and efficiency of our PPF-based LIO in real-world scenarios.
We test on 5 public datasets, \textit{i.e.} \texttt{UTBM} robocar dataset~\cite{eu_longterm_dataset}, \texttt{KITTI}~\cite{geiger2012we}, \texttt{LIOSAM}~\cite{shan2020lio}, \texttt{ULHK}~\cite{wen2020urbanloco} and \texttt{NCLT}~\cite{ncarlevaris-2015a}, including 22 sequences in total.
See Appendix for benchmark details.

\subsection{Comparative Study of PPF and \textit{k}NN}\label{subsec:simulated-experiment}
In this section, we verify the conclusion drawn in Fig.~\ref{fig:compare_knn_plane}: \textit{The PPF takes fewer iterations to converge than \textit{k}NN, and costs at least an order of magnitude less time in every single iteration}.

We generate planes with Gaussian white noise, which, in the real world, may come from LiDAR sampling, cloud de-skewing with noisy IMU measurements, etc.
To simulate the agent's ego-motion, a distance is set between the two planes.
The goal is to register these two scans such that different scans of the same object can overlap in the mapping.
\begin{figure}[h]
    \centering\includegraphics[width=\columnwidth]{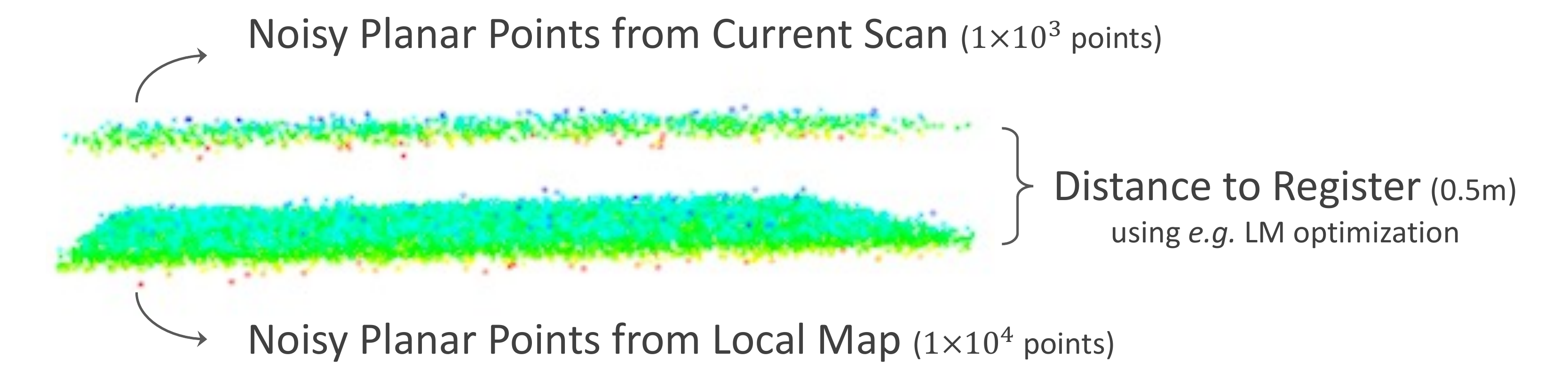}
    \caption{Simulated experiment setup. We generate planes with Gaussian white noise of varying $stddev$. The registration is done with different point matching strategies using the Levenberg-Marquardt algorithm~\cite{more1978levenberg}.}
    \label{fig:sim_plane}
\end{figure}

Results are shown in Fig.~\ref{fig:sim_exp_res}.
We first plot the distance between two planes \textit{w.r.t.} the number of iterations in (a).
For the PPF strategy, regardless of the Gaussian noise $stddev$, the registrations are almost done only after the first iteration (\textit{i.e.} distance\,$\approx$\,$0\mathrm{m}$, overlap).
While for the \textit{k}NN strategy, as the noise grows, the registrations of the first few iterations perform poorly.
This is because with the increase of noise $std$, the 4$\sim$5 NN become unstable and hence the fitted plane gradually deviates from the central plane, thereby misleading the optimizing objective.
Note that applying PCA refinement on \textit{k}NN-fitted planes can hardly reduce the negative effects of noise, because the samples for fitting (\textit{k}) are too few, resulting in the principal component being sensitive to any slight deviations caused by the noise.
Next, we show the number of iterations to converge\footnote[3]{~The convergence criterion is the same as that of LIO-SAM~\cite{shan2020lio}.} \textit{w.r.t.} the noise $std$ in (b).
\begin{figure}[h]
    \centering\includegraphics[width=\columnwidth]{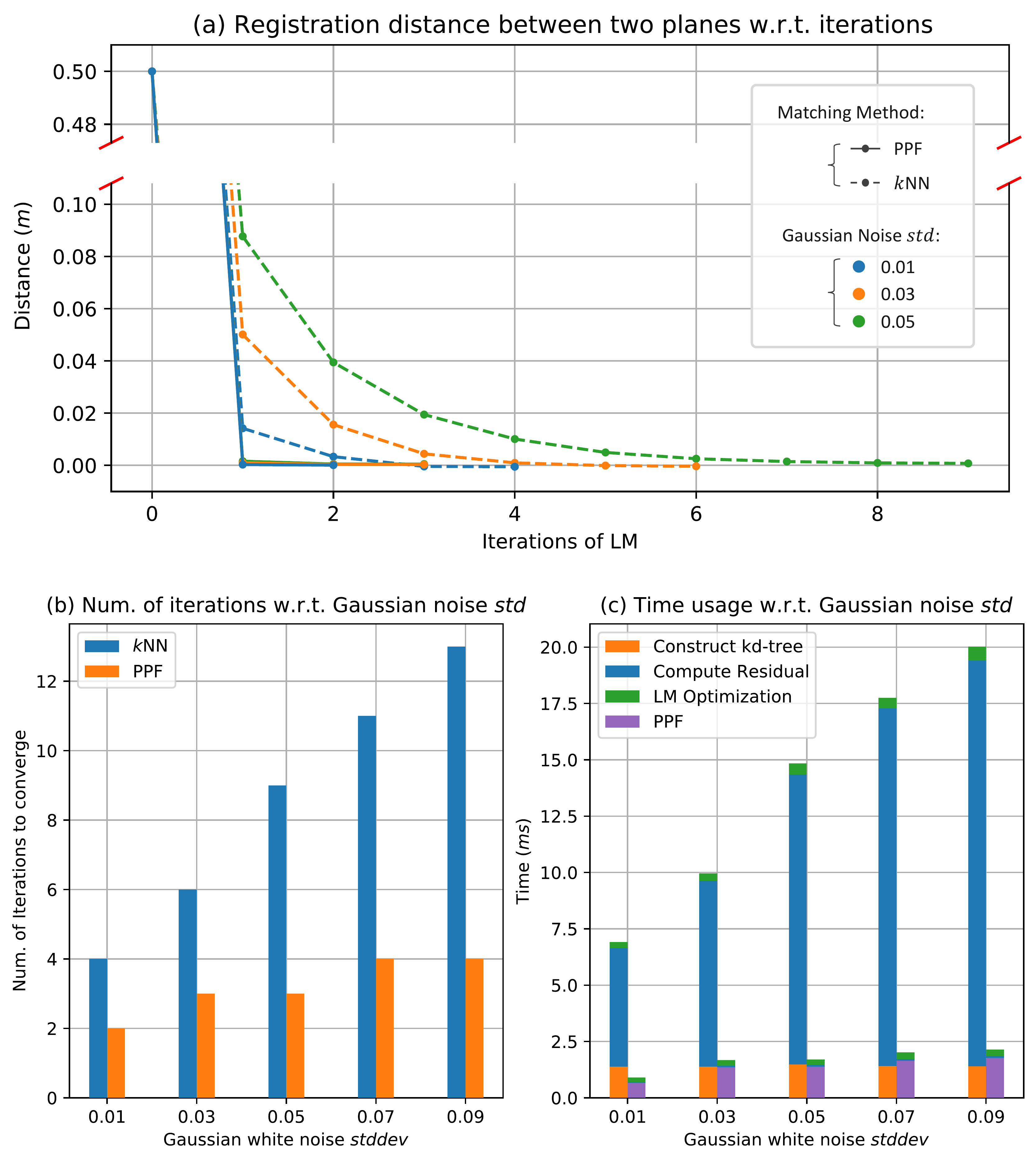}
    \caption{Results of the simulation experiments. \textbf{(c)} left bar: \textit{k}NN; right bar: PPF. All experiments run in a single thread.}
    \label{fig:sim_exp_res}
\end{figure}
Generally, PPF is more robust to noise, and the number of iterations needed increases at a much lower rate.
We further compare the detailed runtime of the two strategies in (c).
The overall efficiency of PPF is significantly higher.
The \textit{k}NN strategy spends most of the time in residual computation, which includes the expensive \textit{k}NN search and plane fitting for each point.
Again, these operations are redundant and can be saved if searching happens at the plane level.
As a result, the residual computing time of PPF is at least one order of magnitude less than that in \textit{k}NN.
From the trend of the overall runtime \textit{w.r.t.} noise, \textit{k}NN grows at a much faster speed, whereas the PPF keeps at a stable level.
The time growth of \textit{k}NN stems mainly from the number of iterations, and the growth of PPF comes only from pre-fitting.
This is because the iPCA refinement takes more time to converge when the noise is large.
However, once the skeleton plane is determined, the LM optimization and residual calculation are no longer affected by the noise of point clouds.

\begin{table*}[htb]
    \centering
    \caption{Time ($\mathrm{ms}$) and Memory Consuming of Each Step in LIO-SAM~\cite{shan2020lio} Using \textit{k}NN or PPF}
    \begin{threeparttable}
        \begin{tabular}{r | cc | cc | cc | cc | cc || cc || cc}
            \toprule
            \multirow{2}{*}{ Dataset } & \multicolumn{2}{c}{ PPF } & \multicolumn{2}{c}{ Build Local Map } & \multicolumn{2}{c}{ Build kd-tree } & \multicolumn{2}{c}{ Scan Match } & \multicolumn{2}{c}{ Calc. Residual } & \multicolumn{2}{c}{ Total time$^1$ } & \multicolumn{2}{c}{ Local Map Size$^2$ } \\
            & \textit{k}NN & PPF & \textit{k}NN & PPF & \textit{k}NN & PPF & \textit{k}NN & PPF & \textit{k}NN & PPF & \textit{k}NN & PPF & \textit{k}NN & PPF \\
            \midrule
            \texttt{LIOSAM-3} & \textbf{0} & 2.64 & 21.28 & \textbf{10.48} & 8.43 & \textbf{5.51} & 19.97 & \textbf{13.84} & 1.50 & \textbf{0.69} & 45.12 & \textbf{30.80} & 66.9k & \textbf{44.1k} \\
            \texttt{LIOSAM-4} & \textbf{0} & 3.16 & 28.12 & \textbf{13.65} & 9.56 & \textbf{6.25} & 21.99 & \textbf{15.96} & 1.66 & \textbf{0.76} & 54.36 & \textbf{36.98} & 75.8k & \textbf{49.9k} \\
            \texttt{LIOSAM-5} & \textbf{0} & 3.07 & 18.20 & \textbf{9.97} & 8.55 & \textbf{6.05} & 22.03 & \textbf{17.51} & 2.01 & \textbf{0.97} & 44.37 & \textbf{34.69} & 67.9k & \textbf{48.3k} \\
            \texttt{LIOSAM-2} & \textbf{0} & 3.61 & 17.69 & \textbf{10.33} & 8.88 & \textbf{6.36} & 23.36 & \textbf{18.55} & 2.11 & \textbf{1.14} & 44.95 & \textbf{36.39} & 70.4k & \textbf{50.5k} \\
            \texttt{KITTI-7} & \textbf{0} & 1.30$^\dag$ & 55.08 & \textbf{21.56} & 8.20 & \textbf{3.72} & 42.93 & \textbf{19.47} & 3.95 & \textbf{0.83} & 112.66 & \textbf{58.73} & 67.6k & \textbf{29.4k} \\
            \texttt{KITTI-2} & \textbf{0} & 1.96$^\dag$ & 36.76 & \textbf{13.95} & 5.19 & \textbf{2.61} & 35.63 & \textbf{16.99} & 3.80 & \textbf{0.99} & 91.25 & \textbf{51.77} & 41.6k & \textbf{15.7k} \\
            \texttt{KITTI-3} & \textbf{0} & 1.50$^\dag$ & 44.55 & \textbf{24.30} & 4.30 & \textbf{2.75} & 28.96 & \textbf{17.48} & 2.49 & \textbf{0.78} & 93.05 & \textbf{62.82} & 37.0k & \textbf{19.5k} \\
            \texttt{KITTI-4} & \textbf{0} & 1.25$^\dag$ & 45.86 & \textbf{21.92} & 4.37 & \textbf{2.32} & 28.47 & \textbf{15.23} & 2.95 & \textbf{0.73} & 93.96 & \textbf{58.02} & 37.3k & \textbf{17.0k} \\
            \texttt{KITTI-6} & \textbf{0} & 1.48$^\dag$ & 43.29 & \textbf{17.32} & 5.48 & \textbf{2.34} & 30.96 & \textbf{15.15} & 4.12 & \textbf{0.90} & 93.23 & \textbf{52.93} & 44.3k & \textbf{15.9k} \\
            \texttt{KITTI-5} & \textbf{0} & 1.22$^\dag$ & 55.24 & \textbf{23.11} & 7.31 & \textbf{3.44} & 43.17 & \textbf{21.10} & 3.52 & \textbf{0.73} & 114.89 & \textbf{64.43} & 60.8k & \textbf{26.6k} \\
            \texttt{KITTI-1} & \textbf{0} & 1.29$^\dag$ & 45.04 & \textbf{22.79} & 4.70 & \textbf{2.69} & 32.25 & \textbf{17.98} & 3.07 & \textbf{0.80} & 97.18 & \textbf{61.96} & 40.3k & \textbf{19.8k} \\

            \bottomrule
        \end{tabular}
        \vspace{0.03cm}
        \begin{tablenotes}
            \scriptsize
            \item[$^1$] The total time includes some common steps \textit{e.g.} de-skewing (not listed) whose timings are unchanged - we only tabulate steps which are sped up.
            \item[$^2$] The local map denotes the one used in the scan-matching algorithm, not the reconstructed map.
            We only decrease the memory usage of LIO algorithms but do not decrease the quality of the mapping.
            The size is measured by the number of points.
            \item[$^\dag$] The \textit{incremental} fitting is enabled for \texttt{KITTI} dataset, to show the benefits for handling the large number of points obtained by Velodyne HDL-64E Laser scanner.
        \end{tablenotes}
    \end{threeparttable}

    \label{tab:time_liosam}
\end{table*}

\begin{table}[h]
    \scriptsize
    \centering
    \caption{Accuracy Comparison of LIO-SAM~\cite{shan2020lio} Using \textit{k}NN and PPF}
    \begin{threeparttable}
        \begin{tabular}{lcccc}
            \toprule
            \multirow{2}{*}{ Dataset } & \multicolumn{2}{c}{ LIO-SAM } & \multicolumn{2}{c}{ LIO-SAM \textcolor{aug_our_method_color}{w/ PPF} } \\
            & APE $(\mathrm{m})$ & RPE $(\mathrm{m})$ & APE $(\mathrm{m})$ & RPE $(\mathrm{m})$ \\
            \midrule
            \texttt{LIOSAM-1} & 0.739 & 6.928 & \textbf{0.719} & \textbf{6.742} \\
            \texttt{LIOSAM-2} & 9.785 & 14.387 & \textbf{8.364} & \textbf{13.353} \\
            \texttt{LIOSAM-5} & 1.365 & \textbf{2.620} & \textbf{1.280} & 2.624 \\
            \texttt{KITTI-1}$^\dag$ & 2.039 & 0.105 & \textbf{2.026} & \textbf{0.101} \\
            \texttt{KITTI-2}$^\dag$ & 1.551 & 1.336 & \textbf{1.547} & \textbf{1.335} \\
            \texttt{KITTI-3}$^\dag$ & \textbf{1.209} & 0.344 & 1.398 & \textbf{0.332} \\
            \texttt{KITTI-4}$^\dag$ & 1.268 & \textbf{0.152} & \textbf{0.960} & 0.162 \\
            \texttt{KITTI-6}$^\dag$ & 0.342 & 0.055 & \textbf{0.328} & \textbf{0.048} \\
            \texttt{KITTI-7}$^\dag$$^*$ & 4.463 & 1.471 & \textbf{3.147} & \textbf{1.461} \\

            \bottomrule
        \end{tabular}
        \begin{tablenotes}
            \scriptsize
            \item[$^{\mathrm{*}}$] Since LIO-SAM requires high-frequency IMU, the Rosbag of KITTI Odom is generated from the corresponding KITTI Raw data.
        \end{tablenotes}
    \end{threeparttable}
    \label{tab:performance_liosam}
\end{table}

\begin{figure}[hbt]
    \centering\includegraphics[width=\columnwidth]{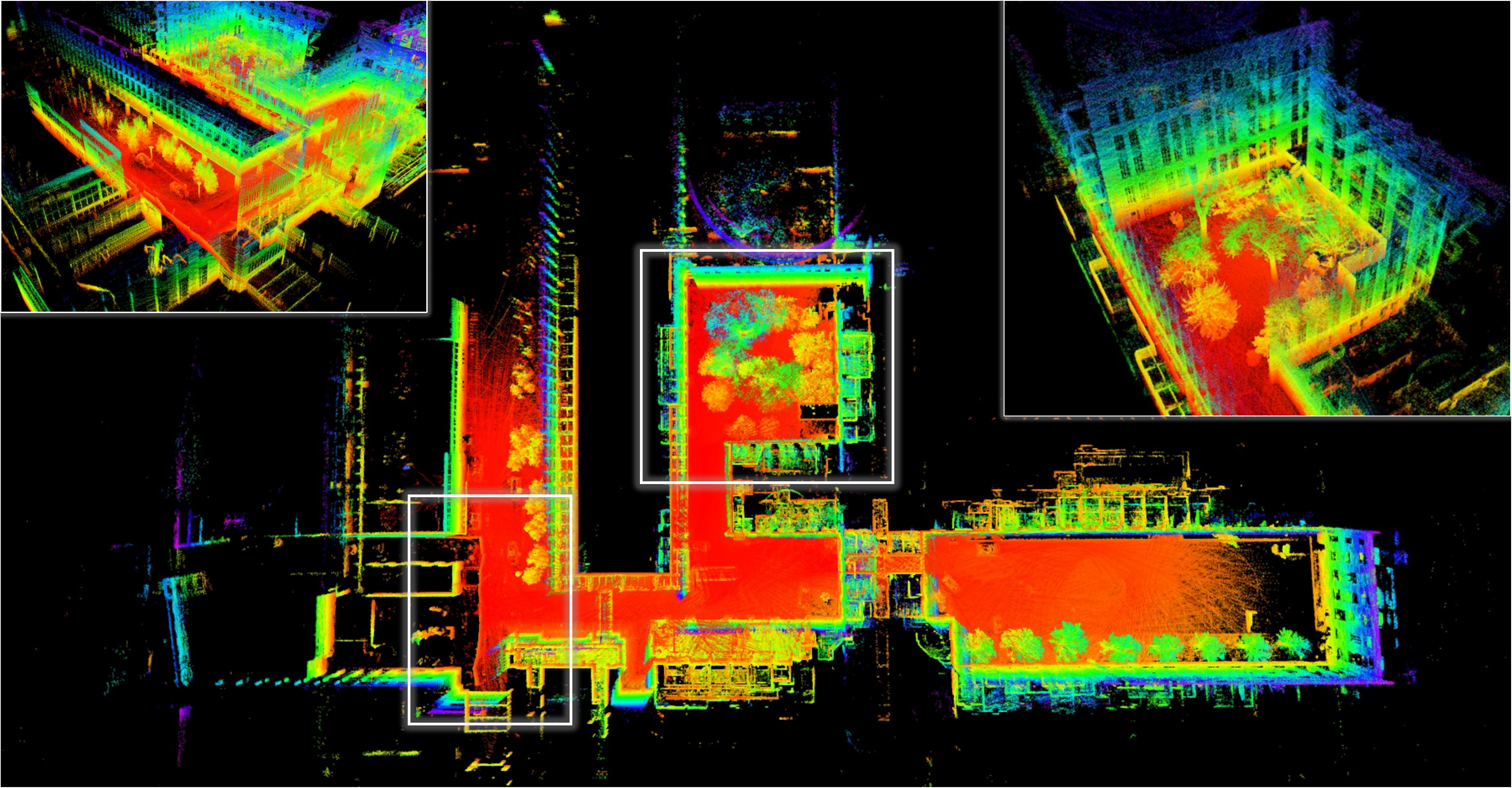}
    \caption{Mapping results of the MIT campus using \texttt{LIOSAM-3} sequence.}
    \label{fig:map_garden}
\end{figure}
\subsection{Efficiency Evaluation}\label{subsec:efficiency_evaluation}
LIO systems' efficiencies are evaluated by analyzing the runtime of each step.
In order to investigate how LIO systems can benefit from the PPF, we first redevelop LIO-SAM~\cite{shan2020lio} and compare the efficiency using \textit{k}NN and PPF.
We tabulate the detailed time cost of each step in Tab.~\ref{tab:time_liosam}.
In the \texttt{LIOSAM} datasets, the overall efficiency (FPS, $\frac{1}{total\:time}$) is increased to 140\%\,$\sim$\,150\%.
In the \texttt{park} sequence, the basic skeleton is not clear due to the large amount of vegetation (tree leaves, grass), hence the FPS is only increased to 128\%.
Specifically, since skeleton points are all extracted out and tracked independently, the efficiencies of local map construction and kd-tree building are greatly boosted, with only an extra cost of $3\,\mathrm{ms}$ from pre-fitting.
In terms of the searching space, the basic skeleton is stored in a sliding window $\mathbf{F}$ with at most 5 plane coef $\in \mathbb{R}^4$, which is much smaller than the local point map, and the expensive \textit{k}NN search is totally abandoned here.
Consequently, time to calculate residuals is reduced by more than half.
To better show the benefit of PPF, we disable the point cloud down-sampling in \texttt{KITTI} dataset, and make use of all points available.
This also helps LIO to reconstruct more fine-grained 3D maps.
The overall FPS is increased more significantly in \texttt{KITTI}, which is up to 1.92$\times$.
Note that there are even two sequences that cannot run in real time without PPF.
The residual calculation's efficiency can also be increased to more than 4$\times$ faster.
The local map memory usage is reduced to a minimum of 36\%.
Meanwhile, because of the heavy Velodyne HDL-64E Laser scanner of \texttt{KITTI}, we enable the \textit{incremental} fitting, which further decreases the time of PPF from $\sim6\,\mathrm{ms}$ (not tabulated) to $\sim1.5\,\mathrm{ms}$.

\begin{figure}[hbt]
    \centering\includegraphics[width=\columnwidth]{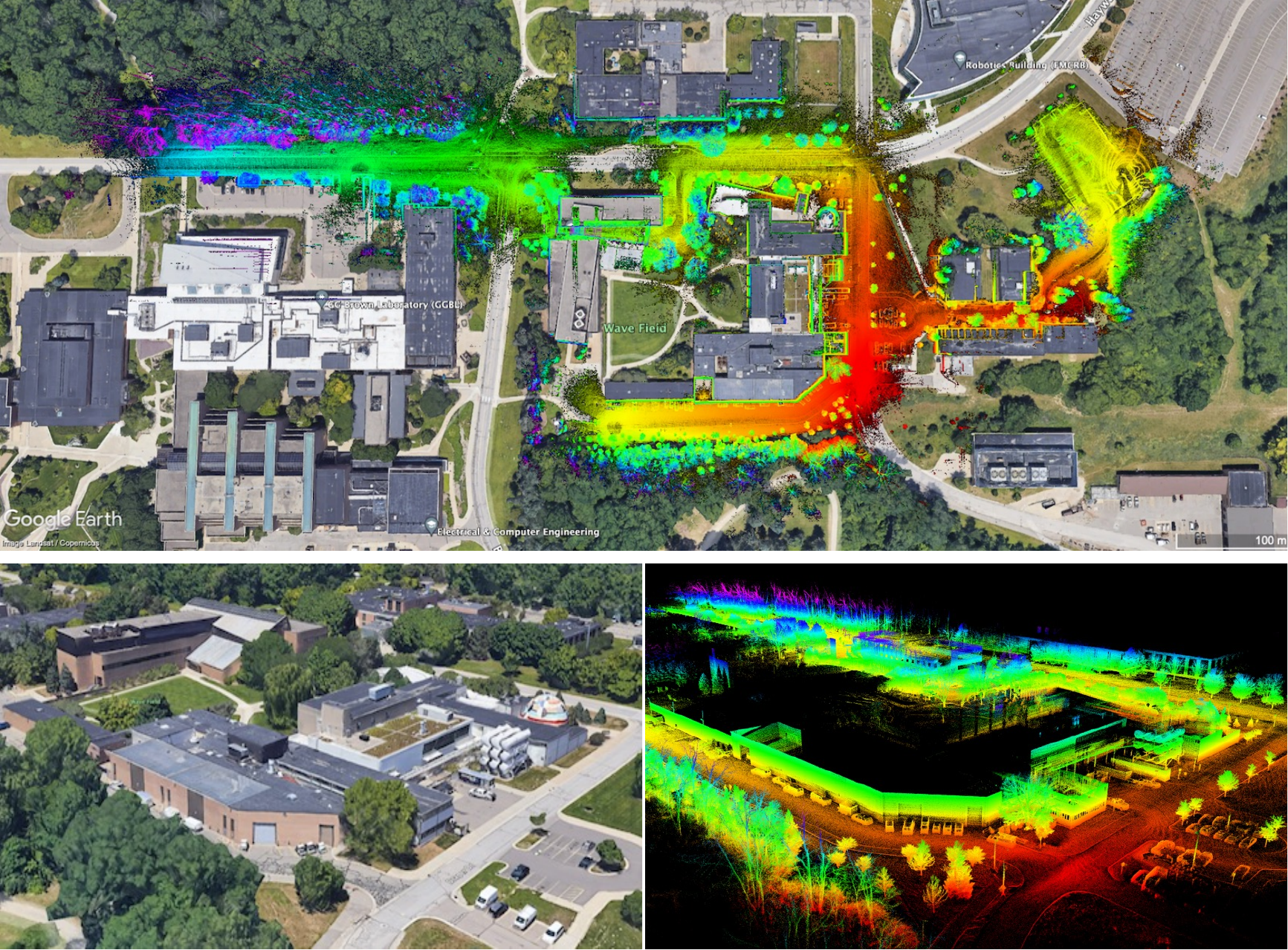}
    \caption{Estimated trajectory and reconstructed maps of a large-scale scene in the University of Michigan North Campus, aligned with Google Earth.}
    \label{fig:map_nclt_20130110}
\end{figure}

\subsection{Accuracy Evaluation}\label{subsec:accuracy-evaluation}
The accuracy is measured by root mean square error (RMSE) of absolute pose error (APE) metric and relative pose error (RPE) metric, which were calculated by \textit{evo}~\cite{grupp2017evo} in our experiments.
We first compare the accuracy of our algorithm with LIO-SAM, as reported in Tab.~\ref{tab:performance_liosam}.
In \texttt{LIOSAM} dataset, 3 sequences are with GPS measurements, which are treated as ground truth trajectories.
Since LIO-SAM requires high-frequency IMU measurements, all \texttt{KITTI} sequences are generated from KITTI-RAW dataset using $kitti\_to\_bag$\footnote[4]{~\url{https://github.com/tomas789/kitti2bag}} toolkit.
We also show the accuracy results of integrating PPF into other state-of-the-art LIO systems, \textit{e.g.} Faster-LIO~\cite{bai2022faster}.
The accuracies of 4 systems are compared in 9 sequences in Tab.~\ref{tab:performance_fastlio}.
Since LIO-SAM requires a 9-axis IMU, which is not available in the \texttt{UTBM} dataset, the results are not shown.
Some sequences provide the ground truth poses at a low frequency (\textit{e.g.} \texttt{1\,hz} in \texttt{ULHK}), therefore, we adjust the \textit{t\_max\_diff} parameter of \textit{evo} to increase the maximum timestamp difference for data association.
As can be seen, with \textit{k}NN replaced by PPF, both LIO-SAM and Faster-LIO can reach a generally comparable accuracy to their original counterparts, with less time and memory usage.
Visualizations of estimated trajectories are compared in Fig.~\ref{fig:evo_traj}, which also shows that the ones with PPF keep the same level of accuracy in most sequences (sometimes even more visually accurate, \textit{e.g.} the \texttt{KITTI-7} sequence).

\begin{table*}[htb]
    \scriptsize
    \centering
    \caption{Accuracy Comparison of Existing State-of-the-Art Methods}
    \begin{tabular}{lccccccccc}
        \toprule
        \multirow{2}{*}{ Dataset } & \multicolumn{2}{c}{ Faster-LIO~\cite{bai2022faster} } & \multicolumn{2}{c}{ Faster-LIO \textcolor{aug_our_method_color}{w/ PPF} } & \multicolumn{2}{c}{ LIO-SAM~\cite{shan2020lio} } & \multicolumn{2}{c}{ LiLi-OM~\cite{liliom2021} } & \multirow{2}{*}{ Distance (km) } \\
        & APE $(\mathrm{m})$ & RPE $(\%)$ & APE $(\mathrm{m})$ & RPE $(\%)$ & APE $(\mathrm{m})$ & RPE $(\%)$ & APE $(\mathrm{m})$ & RPE $(\%)$ &  \\
        \midrule
        \texttt{NCLT-1} & \textbf{1.239} & \textbf{0.157} & 1.365 & \textbf{0.157} & 10.095 & 1.307 & - & - & 1.14 \\
        \texttt{NCLT-2} & 1.478 & \textbf{0.153} & \textbf{1.230} & \textbf{0.153} & 26.351 & 1.309 & - & - & 3.19 \\
        \texttt{UTBM-1} & 26.519 & 1.806 & 27.470 & \textbf{1.800} & - & - & \textbf{14.950} & 5.669 & 6.40 \\
        \texttt{UTBM-6} & 15.989 & 0.876 & 16.857 & \textbf{0.866} & - & - & \textbf{12.954} & 4.124 & 5.14 \\
        \texttt{UTBM-3} & \textbf{15.510} & 0.482 & 16.502 & \textbf{0.474} & - & - & 15.989 & 2.869 & 4.98 \\
        \texttt{UTBM-2} & 15.002 & \textbf{1.983} & \textbf{14.996} & 1.985 & - & - & 24.466 & 12.051 & 4.99 \\
        \texttt{UTBM-4} & 14.560 & 2.036 & 15.701 & \textbf{2.034} & - & - & \textbf{13.654} & 12.735 & 4.99 \\
        \texttt{ULHK-1} & \textbf{0.972} & 0.915 & 1.038 & \textbf{0.631} & 1.503 & 0.931 & \textbf{0.964} & 0.635 & 0.60 \\
        \texttt{ULHK-2} & \textbf{0.888} & 0.443 & 0.897 & 0.421 & 0.926 & 0.465 & 0.893 & \textbf{0.417} & 0.62 \\

        \bottomrule
    \end{tabular}
    \label{tab:performance_fastlio}
\end{table*}

\begin{table}[h]
    \scriptsize
    \centering
    \caption{Runtime ($\mathrm{ms}$) Comparison of Fast-LIO2~\cite{xu2022fast} Using \textit{k}NN and PPF}
    \begin{threeparttable}
        \begin{tabular}{lcccc}
            \Xhline{0.8pt}
            Dataset & \begin{tabular}[c]{@{}c@{}}Calc.\\ Residual \end{tabular} & \begin{tabular}[c]{@{}c@{}}Map \\Incremental \end{tabular} &
            IEKF
            & \begin{tabular}[c]{@{}c@{}}Map Size$^1$ \end{tabular} \\

            \hline \texttt{NCLT-2} & 2.79 & 1.52 & 12.67 & 3.37k \\
            \textcolor{aug_our_method_color}{w/ PPF} & \textbf{1.90} & \textbf{1.26} & \textbf{8.73} & \textbf{2.33k}  \\

            \hline \texttt{NCLT-1} & 2.54 & 1.56 & 11.66 & 2.74k \\
            \textcolor{aug_our_method_color}{w/ PPF} & \textbf{1.72} & \textbf{1.28} & \textbf{7.98} & \textbf{1.89k}  \\

            \hline \texttt{UTBM-5} & 2.90 & 3.66 & 13.86 & 3.29k \\
            \textcolor{aug_our_method_color}{w/ PPF} & \textbf{1.57} & \textbf{2.62} & \textbf{7.95} & \textbf{1.81k}  \\

            \hline \texttt{UTBM-2} & 2.88 & 3.34 & 13.49 & 3.34k \\
            \textcolor{aug_our_method_color}{w/ PPF} & \textbf{1.83} & \textbf{2.64} & \textbf{9.06}  &  \textbf{2.14k}  \\

            \hline \texttt{UTBM-4} & 2.68 & 3.16 & 12.64 & 3.30k \\
            \textcolor{aug_our_method_color}{w/ PPF} & \textbf{1.81} & \textbf{2.58} & \textbf{8.87} & \textbf{2.14k}  \\

            \hline \texttt{UTBM-3} & 2.74 & 3.39 & 12.97 & 3.29k \\
            \textcolor{aug_our_method_color}{w/ PPF} & \textbf{1.76} & \textbf{2.61} & \textbf{8.71}  &  \textbf{2.08k}  \\

            \hline \texttt{UTBM-6} & 2.80 & 3.27 & 13.33 & 3.31k \\
            \textcolor{aug_our_method_color}{w/ PPF} & \textbf{1.70} & \textbf{2.45} & \textbf{8.57} & \textbf{2.07k}  \\

            \hline \texttt{ULHK-1} & 2.87 & 2.08 & 13.64 & 2.79k \\
            \textcolor{aug_our_method_color}{w/ PPF} & \textbf{1.37}  &  \textbf{1.47} & \textbf{7.39} & \textbf{1.27k}  \\

            \hline \texttt{ULHK-2} & 2.96 & 2.01 & 14.03 & 2.83k \\
            \textcolor{aug_our_method_color}{w/ PPF} & \textbf{1.44}  &  \textbf{1.45} & \textbf{7.65} & \textbf{1.30k}  \\
            \Xhline{0.8pt}
        \end{tabular}
        \vspace{0.03cm}
        \begin{tablenotes}
            \scriptsize
            \item[$^{\mathrm{1}}$] Map memory is measured by the average number of new points in each scan that need to be maintained by the ikd-tree.
        \end{tablenotes}
    \end{threeparttable}
    \label{tab:time_fastlio}
\end{table}

\begin{figure}[h]
    \centering\includegraphics[width=\columnwidth]{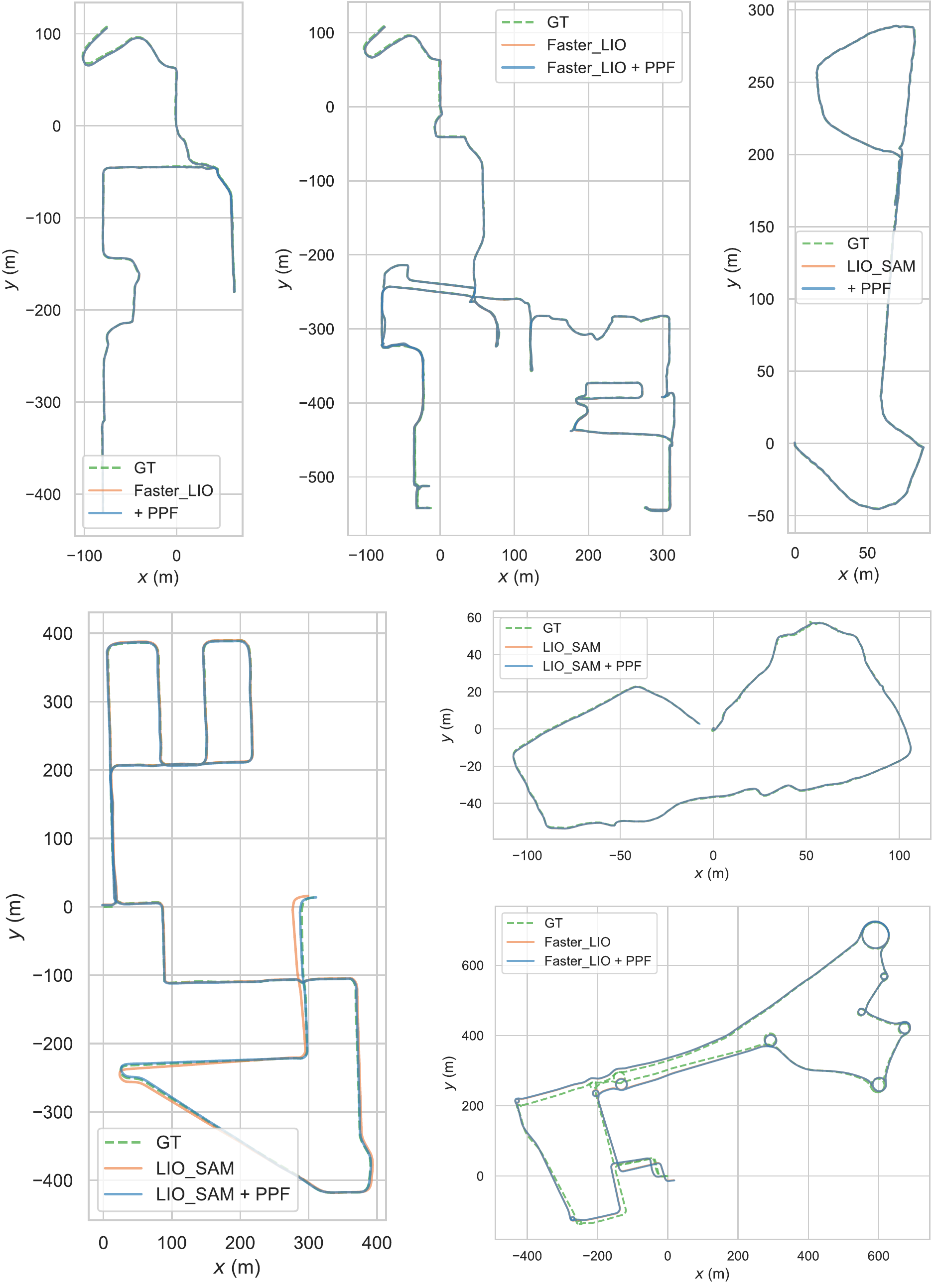}
    \caption{The estimated and GT trajctories. From top left to bottom right: \texttt{NCLT-1}, \texttt{NCLT-2}, \texttt{LIOSAM-5}, \texttt{KITTI-7}, \texttt{LIOSAM-1} and \texttt{UTBM-5}.}
    \label{fig:evo_traj}
\end{figure}

\section{Conclusion}\label{sec:conclusion}

We investigate the computation of LiDAR-Inertial odometry systems, and reveal the redundancy of most \textit{k}NN searches, plane fittings and the un-necessity of the large local map.
To achieve more efficient LiDAR scan tracking, we design a plane pre-fitting pipeline to achieve the basic-skeleton tracking.
We demonstrate the un-robustness of \textit{k}NN to noisy and non-strict planes, and introduce iPCA to PPF to handle this.
Moreover, skeleton planes can be incrementally updated by leveraging IMU measurements and the sequential nature of LiDAR scans to further speed up PPF.
For complex scenarios where planes can have small included angles, a simple yet effective sandwich layer is introduced to eliminate false point-to-plane matches.
Experiments on 22 sequences across 5 open datasets show that, by contrast, our skeleton tracking consumes only a minimum of 36\% of the original local map size, achieving 4$\times$ faster in residual computing and up to 1.92$\times$ overall FPS, while maintaining the same level of accuracy.
To make contributions to the SLAM community, we publish the full C++ implementation of this paper.


\section*{APPENDIX}

\begin{table}[h]
    \scriptsize
    \centering
    \caption{Details of All the 22 Sequences Used in Sec.~\ref{sec:experiments}}
    \begin{threeparttable}
        \begin{tabular}{llcc}
            \toprule
            & \begin{tabular}[c]{@{}c@{}}Sequence\end{tabular} & \begin{tabular}[c]{@{}c@{}}Duration\\$(\mathrm{min:sec})$ \end{tabular} & \begin{tabular}[c]{@{}c@{}}Distance\\$(\mathrm{km})$ \end{tabular} \\
            \midrule

            \texttt{LIOSAM-1} & \texttt{Campus\,(small)} & 6:47 & 0.55 \\
            \texttt{LIOSAM-2} & \texttt{Campus\,(large)} & 16:34 & 1.44 \\
            \texttt{LIOSAM-3} & \texttt{Walking} & 10:55 & 0.80 \\
            \texttt{LIOSAM-4} & \texttt{Garden} & 6:02 & 0.46 \\
            \texttt{LIOSAM-5} & \texttt{Park} & 9:20 & 0.65 \\

            \texttt{NCLT-1} & \texttt{2013\_01\_10} & 17:05 & 1.14 \\
            \texttt{NCLT-2} & \texttt{2012\_04\_29} & 43:19 & 3.19 \\

            \texttt{UTBM-1} & \texttt{2019\_01\_31} & 16:00 & 6.40 \\
            \texttt{UTBM-2} & \texttt{2018\_07\_17} & 15:59 & 4.99 \\
            \texttt{UTBM-3} & \texttt{2018\_07\_19} & 15:26 & 4.98 \\
            \texttt{UTBM-4} & \texttt{2018\_07\_20} & 16:45 & 4.99 \\
            \texttt{UTBM-5} & \texttt{2019\_04\_12 roundabout} & 12:10 & 4.86 \\
            \texttt{UTBM-6} & \texttt{2019\_04\_18} & 14:55 & 5.14 \\

            \texttt{ULHK-1} & \texttt{2019\_01\_17} & 5:18 & 0.60 \\
            \texttt{ULHK-2} & \texttt{2019\_03\_17} & 5:18 & 0.62 \\

            \texttt{KITTI-1} & \texttt{Raw 2011\_09\_30\_drive\_0034} & 2:07 & 0.92 \\
            \texttt{KITTI-2} & \texttt{Raw 2011\_09\_26\_drive\_0002} & 0:08 & 0.08 \\
            \texttt{KITTI-3} & \texttt{Raw 2011\_09\_26\_drive\_0061} & 1:12 & 0.49 \\
            \texttt{KITTI-4} & \texttt{Raw 2011\_09\_26\_drive\_0064} & 0:59 & 0.44 \\
            \texttt{KITTI-5} & \texttt{Raw 2011\_09\_30\_drive\_0018} & 4:48 & 2.21 \\
            \texttt{KITTI-6} & \texttt{Raw 2011\_09\_26\_drive\_0084} & 0:40 & 0.24 \\
            \texttt{KITTI-7} & \texttt{Odometry Seq.\,08} & 8:56 & 3.22 \\

            \bottomrule
        \end{tabular}
    \end{threeparttable}
    \label{tab:details_seq}
\end{table}

\bibliographystyle{IEEEtran}
\bibliography{ref}
\end{document}